\documentclass[11pt]{article}
\usepackage[letterpaper, total={6in, 8in}]{geometry}

\usepackage{amsmath}
\usepackage{mathtools, amssymb, amscd, amsthm, amsfonts}
\usepackage{nccmath}

\usepackage{graphicx}
\usepackage{adjustbox}
\usepackage{hyperref}
\usepackage[capitalise]{cleveref}

\usepackage{calc} 

\usepackage{multirow} 
\usepackage{setspace} 
\usepackage{algorithm, algorithmic} 

\usepackage{natbib}
\newcommand{\codeindent}{\hphantom{~~~~~}}

\usepackage{dashbox}
\newcommand\dashedph[1][H]{\setlength{\fboxsep}{0pt}\setlength{\dashlength}{1.5pt}\setlength{\dashdash}{0.5pt} \dbox{\phantom{#1}}}

\usepackage[dvipsnames]{xcolor}

\usepackage{mdframed}

\DeclareMathOperator*{\argmin}{\arg\min}

\newtheorem{theorem}{Theorem}[section]

\newtheorem{defn}{Definition}[section]
\newtheorem{cor}{Corollary}[section]
\newtheorem{prop}{Proposition}[section]
\newtheorem{assumption}{Assumption}[section]
\newtheorem*{remark}{Remark}

\title{Decentralized Multi-Agent Reinforcement Learning for Continuous-Space Stochastic Games\thanks{This manuscript is the expanded version of an article appearing at the American Control Conference, 2023.} 
}
\author{
Awni Altabaa\thanks{\small{Department of Statistics and Data Science at Yale University, New Haven, CT, United States:{\tt awni.altabaa@yale.edu}}} \and
Bora Yongacoglu\thanks{Department of Mathematics and Statistics at Queen's University, Kingston ON, Canada: {\tt\small 1bmy@queensu.ca, resp. yuksel@queensu.ca}} \and
and Serdar Y\"uksel\footnotemark[3]
}
\date{}

\begin{document}

\maketitle

\begin{abstract}
Stochastic games are a popular framework for studying multi-agent reinforcement learning (MARL). Recent advances in MARL have focused primarily on games with finitely many states. In this work, we study multi-agent learning in stochastic games with general state spaces and an information structure in which agents do not observe each other's actions. In this context, we propose a decentralized MARL algorithm and we prove the near-optimality of its policy updates. Furthermore, we study the global policy-updating dynamics for a general class of best-reply based algorithms and derive a closed-form characterization of convergence probabilities over the joint policy space. 
\end{abstract}


\section{Introduction}

Multi-agent reinforcement learning (MARL) is the study of the learning dynamics of strategic agents that coexist in a shared environment, and is one of the important frontiers of machine learning and control. In this paper, we study MARL in \emph{stochastic games}, also known as Markov games, a multi-agent generalization of Markov decision problems (MDPs) in which the cost-relevant history of the system is summarized by a state variable \cite{shapleyStochasticGames1953}. Due to its ability to model both dynamic inter-temporal choice as well as strategic interaction, the stochastic games model has long been a popular framework for studying multi-agent learning \cite{littmanMarkovGames1994}.

In comparison to single-agent reinforcement learning, analysis of MARL is difficult due to several challenges inherent to multi-agent systems, including non-stationarity, conflicting interests, and decentralized information. As a result, fundamental understanding of multi-agent reinforcement learning theory has lagged behind its single-agent counterpart \cite{zhangMultiagentReinforcement2021}. Owing partly to a variety of impressive empirical successes, such as those of \cite{silver2016mastering} and \cite{vinyalsAlphastarMastering2019}, the subject of learning in stochastic games has recently attracted considerable attention, and theoretical advances have been made along several research axes. In \cref{sec:related_work}, we survey some of these advances in greater detail, and we observe that the bulk of theoretical progress in MARL has focused on games with finitely many states. By contrast, considerably less is known about learning in games with more general state spaces, which model various real-world settings that do not admit finite state space models.

In this paper, we study the problem of decentralized multi-agent reinforcement learning in stochastic games with standard Borel state spaces, which includes finite sets, countably infinite sets, and finite dimensional vector spaces as special cases. We are interested in the setting in which a generative model is not available for sampling state and (joint) action feedback and in which there is no offline data set available to the players. Instead, we focus on a setting in which players must learn using in-game feedback, updating their estimates and adjusting their behaviour only as new information arises from sequential interaction with the environment.  We approach the problem with an eye on decentralized settings, in which agents may not be able to observe the actions of other agents, but must nevertheless attempt to respond optimally to the decision rules of others. In this context, we wish to propose an algorithm that is justifiable from an individual rationality point of view, but also comes with some long term guarantee of system-wide stability---with policies converging to some joint policy under self-play.

With these objectives in mind, we present a decentralized MARL algorithm for stochastic games with standard Borel state spaces, and we give a formal analysis of its properties. The algorithm proposed here extends the \emph{Decentralized Q-learning} algorithm of \cite{arslanDecentralizedQlearning2017} to settings with general state spaces by quantizing the state variable. Although this approach is natural, one must implement it with care: by quantizing the state space, the algorithm designer artificially restricts the class of policies and may unintentionally limit the efficacy of its learning agents. We rigorously analyze such a state space quantization, and we provide sufficient conditions on the game that ensure such an approach entails an arbitrarily small loss in performance. To accomplish this analysis, we use recent theoretical advances in single-agent MDPs with general state spaces, building most prominently upon the quantized Q-learning results of \cite{karaQlearningMDPs2021}.

The algorithm presented here relies on the exploration phase technique, which requires that agents change their policies only at prescribed times, $t_0 < t_1 < \dots $, while holding their policies fixed during the intervals $[t_k, t_{k+1} - 1]$, for $k \geq 0$. (That is, players do not change policies between update times.) The update times are predetermined and common to all agents, but outside of this no special coordination structure is assumed. In particular, all agents may change their policies simultaneously at an update time $t_k$ or may choose to not change their policy. Besides this assumption on synchrony, agents need not communicate, observe one another's actions or cost realizations, or even be aware of the other agents or the fact that they are playing a game. Despite this decentralization, we show that policy updates of our algorithm are asymptotically optimal for each agent with respect to the most recently observed environment. 

In addition to the near optimality of policy updates resulting from our algorithm, we also analyze the joint policy dynamics for a class of best-reply-based updates. For such algorithms, we study convergence to equilibrium and give a characterization of the convergence behaviour when the game admits weakly acyclic structure. Finally, we describe additional settings in which self-play of our algorithm drives the joint policy process to $\epsilon$-equilibrium policies.


\subsection*{Contributions:}

\begin{itemize}

    \item In \cref{alg:cts_dec_qlearning}, we propose a decentralized MARL algorithm for environments with general state spaces. In \cref{thm:cts_transition_near_br}, we prove that this algorithm yields policy updates which are near-optimal with respect to each agent's observed environment, under the assumptions of weak continuity of the transition kernel and continuity of the cost function.

    \item In \cref{prop:abs_probs}, we theoretically analyze the global joint policy-updating dynamics of a general class of best-reply-based algorithms and study convergence to equilibria.




    \item In \cref{sec:simulation_study}, we present a simulation study which corroborates our theoretical study. 
    
\end{itemize}

\subsection*{Organization} The remainder of the paper is structured as follows. In \Cref{sec:related_work} we discuss some related work on multi-agent reinforcement learning algorithms with theoretical guarantees. In \Cref{sec:background} we give the relevant background on modeling multi-agent environments. In \Cref{sec:cts_dec_alg}, we propose a continuous-space extension of the decentralized multi-agent Q-learning algorithm and present our accompanying theoretical results. In \Cref{sec:policy_updating_dynamics} we derive a complete characterization of the global policy-updating dynamics of our algorithms. In \cref{sec:convergence_to_eq}, we discuss the topic of convergence to equilibrium and suggest lines of future research.  A simulation study of our algorithm is presented in \Cref{sec:simulation_study}. 

\ 

\noindent \textbf{Notation.} We use $\mathbb{P}$ and $\mathbb{E}$ to denote the probability measure or expectation with subscripts and superscripts to denote the underlying probability space. $\mathcal{P}(S)$ denotes the set of probability measures on $S$ and $\mathcal{P}(S' | S)$ denotes the set of transition kernels from $S$ to $S'$. 
We will tend to use $\hat{\dashedph}$ to denote quantities associated with approximation due to quantization.

\section{Related Work} \label{sec:related_work}

%
%
%
%
%
%

Multi-agent learning is an interdisciplinary subject with contributions from various fields. In recent years, there has been a massive amount of new contributions to the subject, including contributions that are theoretical, empirical, and applied in nature. As such, a comprehensive overview of related literature is beyond the scope of this paper. Therefore, the discussion below takes a narrower view and attempts to focus on works that are closer to the subject matter of this study. For recent surveys of multi-agent reinforcement learning, see \cite{zhangMultiagentReinforcement2021} or \cite{ozdaglar2021independent}. For references to older research on multi-agent learning from other perspectives, \cite{cesa2006prediction,fudenberg1998theory} and \cite{young2004strategic}.

In this paper, we adopt stochastic games as our formal model for multi-agent reinforcement learning. Stochastic games have long been a popular framework for studying MARL due to their ability to model both inter-temporal consequences of actions as well as multi-player strategic considerations \cite{littmanMarkovGames1994}. The $N$-player stochastic game model generalizes both single-agent Markov decision problems (MDPs) as well as single-state repeated games.



In this paper, we study stochastic games with full state observability at each agent and no action-sharing between agents. This paradigm, in which agents have access only to the state variable and their own history of actions and cost realizations, is sometimes called the \emph{independent learning} paradigm, though the terminology is not uniform \cite{claus1998dynamics, laurent2011world, ozdaglar2021independent}. Under various names, independent learners have been the subject of considerable research recently, \cite{sayin2021decentralized,daskalakis2020independent,yongacogluSatisficingPaths2021,yongacogluDecentralizedLearning2022}, and important theoretical advances have been made.

We contrast the so-called independent learning paradigm studied here with the \emph{joint action learner} paradigm, in which each agent observes the actions of all other agents. Under an assumption of perfect monitoring of one another's actions, many theoretical contributions were made for joint action learners. We selectively cite \cite{littmanMarkovGames1994,littmanFriendFoe2001} and \cite{huNashQLearning} as a representative sampling of rigorous work on joint action learners in stochastic games.

Theoretical results for independent learners are comparatively more difficult to obtain than their counterparts for joint action learners, as salient information is hidden from agents. However, the perfect monitoring assumption can be too stringent for many (even mildly) decentralized settings. Moreover, the number of joint actions scales exponentially in the number of agents, meaning that joint action learning algorithms are often intractable.







While the preceding works study stochastic games with finite state spaces, this paper is concerned with a more general class of games that allows for standard Borel state spaces, which include finite state spaces and finite dimensional real vector spaces as special cases.

This work is in the tradition of \emph{regret testing}, pioneered by \cite{foster2006regret} for stateless repeated games and extended to multi-state stochastic games in \cite{arslanDecentralizedQlearning2017}. Other notable contributions in the regret testing tradition include \cite{germano2007global} and \cite{marden2009payoff} for stateless repeated games and \cite{yongacoglu2021decentralized} and \cite{yongacogluDecentralizedLearning2022} for multi-state stochastic games. 

Here, we discuss how one can modify the decentralized Q-learning algorithm of \cite{arslanDecentralizedQlearning2017} to accommodate for a general state space. To accomplish this generalization, we use tools from single-agent reinforcement learning theory. In particular, we build on the work of \cite{saldiAsymptoticOptimality2017}, which studies quantization of state and action spaces in MDPs, and \cite{karaQlearningMDPs2021} (building on \cite{karaConvergenceFinite2021}), which studies Q-learning with quantization in MDPs with general state spaces.

\section{Background on Stochastic Games}
\label{sec:background}


In this paper, we model our multi-agent environment as a \textit{stochastic game}. A (discounted) stochastic game is a discrete-time controlled Markov process which models systems in which multiple agents interact in an environment, each agent pursuing its own objective. 

A stochastic game is described by a tuple 
\[
(\mathcal{N}, \mathbb{X}, \{\mathbb{U}^i\}_{i \in \mathcal{N}}, \mathcal{T}, \{ c^i \}_{i \in \mathcal{N}}, \{ \beta^i \}_{i \in \mathcal{N}}),
\]
where $\mathcal{N} = \{ 1, .., N\}$ is the set of $N$ agents, $\mathbb{X}$ is a standard Borel state space, and for each $i \in \mathcal{N}$, $\mathbb{U}^i$ is agent $i$'s (standard Borel) action space. We denote the joint action space by $\boldsymbol{U} \coloneqq \mathbb{U}^1 \times \cdots \times \mathbb{U}^N$. State transitions are governed by $\mathcal{T}$, a stochastic kernel on $\mathbb{X}$ given  $\mathbb{X} \times \boldsymbol{U}$. Player $i$'s cost is determined by a stage-wise cost function $c^i: \mathbb{X} \times \boldsymbol{U} \to \mathbb{R}$, and player $i$ uses a discount factor $\beta^i \in [0,1)$ to aggregate its stream of stage costs.

\vspace{5pt}

At time $t=0$, the system begins at state $x_0$. At each time $t \geq 0$, each agent $i$ recalls the local history variable, to be described shortly, and selects an action $u_t^i \in \mathbb{U}^i$. After each agent has taken an action, agent $i$ incurs a cost $c^i_t = c^i(x_t, \boldsymbol{u}_t)$, where $\boldsymbol{u}_t = (u^1_t, ..., u^N_t)$. The state randomly evolves according to the transition kernel as $x_{t+1} \sim \mathcal{T} ( \cdot | x_t, \boldsymbol{u}_t )$.


\vspace{5pt}

In the independent learning paradigm, at each time step, agents may observe the global state variable, their own actions, and their own cost realization. The history of these observations may be used by each agent to choose their next action. We call this the the history/information variable, and denote player $i$'s information variable at time $t$ by $I^i_t = ( x_0, u^i_0, c^i_0, \dots, x_{t-1}, u^i_{t-1},c^i_{t-1}, x_t )$. The information variable $I^i_t$ takes values in the set $\mathbb{H}^i_t := \left( \mathbb{X} \times \mathbb{U}^i \times \mathbb{R} \right)^t \times \mathbb{X}$. Note that agents do not observe the actions of other agents. 

\vspace{5pt} 

An \textit{policy} for agent $i$ is a sequence of stochastic kernels $\pi^i = (\pi^i)_{t \geq 0}$ where $\pi^i \in \mathcal{P}(\mathbb{U}^i | \mathbb{H}^i_t)$. When following a policy $\pi^i$, player $i$ selects its action $u^i_t \sim \pi^i_t ( \cdot | I^i_t)$. A policy is said to be \textit{stationary} if its action selection depends only on $x_t$.\footnote{Formally, $\pi^i$ is stationary if for any $t, k \geq 0$ and information variables $I^i_t \in \mathbb{H}^i_t$, $h^i_k \in \mathbb{H}^i_k$, we have that $\pi^i_t ( \cdot | I^i_t ) = \pi^i_k (\cdot | h^i_k )$ whenever $I^i_t$ and $h^i_t$ agree in their final component.} A policy is said to be \textit{deterministic} if the action is a deterministic function of the information variable. 
When the action sets $\{ \mathbb{U}^i \}_i$ are taken to be compact, we define the $\rho$-perturbation of a policy $\pi^i$ to be a randomized policy $\pi^i_\rho  =  (1-\rho) \pi^i + \rho \text{Uniform}(\mathbb{U}^i)$, where $\text{Uniform}(\mathbb{U}^i)$ is the uniform distribution on the action set $\mathbb{U}^i$.

The objective of each agent is to minimize their own expected discounted cumulative cost, which depends on both their policy and the other agents' policies. The joint policy used by the agents  $\boldsymbol{\pi}$ induces a probability measure on the trajectories of $(x_t, \boldsymbol{u}_t)_{t\geq0}$, with associated expectation $\mathbb{E}^{\boldsymbol{\pi}}$. Then the expected discounted cost of agent $i$ is given by
\begin{equation}
    J_{x}^{i}(\boldsymbol{\pi}) \coloneqq \mathbb{E}_{x}^{\boldsymbol{\pi}} \left[ \sum_{t=0}^{\infty}{(\beta^i)^t c^i(X_t, \boldsymbol{U_t})} \right].
    \label{eq:e_cost}
\end{equation}

Each agent's policy can be viewed as a response to the environment they see, which includes all other agents' policies. Let $\boldsymbol{\pi}$ be the joint policy of all agents, then denote agent $i$'s policy by $\pi^i$ and the other agents' policies by $\boldsymbol{\pi}^{-i}$. 

\begin{defn}
    \label{defn: best_reply}
    Let $\epsilon \geq 0$ and let $\Gamma^i$ be a subset of player $i$'s policies. A policy $\pi^{*i} \in \Gamma^i$ is an $\epsilon$-\textit{best-reply} to $\boldsymbol{\pi}^{-i}$ in $\Gamma^i$ if
    \begin{equation}
        J_{x}^{i}(\pi^{*i}, \boldsymbol{\pi}^{-i}) = \inf_{\pi^i \in \Gamma^i}{J_{x}^{i}(\pi^i, \boldsymbol{\pi}^{-i})} + \epsilon, \ \forall x \in \mathbb{X}.
        \label{eq:br}
    \end{equation}
    Furthermore, a 0-best-reply $\pi^{*i}$ to $\boldsymbol{\pi}^{-i}$ is called a \textit{strict best-reply} to $(\pi^{i}, \boldsymbol{\pi}^{-i})$ if $J_{x}^{i}(\pi^{*i}, \boldsymbol{\pi}^{-i}) < {J_{x}^{i}(\pi^i, \boldsymbol{\pi}^{-i})}, \quad \text{for some } x \in \mathbb{X}$
\end{defn}

\begin{defn}
    \label{defn:equilibrium}
    For $\epsilon \geq 0$, a policy $\boldsymbol{\pi}^* \in \boldsymbol{\Gamma}$ is an \textit{$\epsilon$-equilibrium} in $\boldsymbol{\Gamma}$ if $\pi^{*i}$ is an $\epsilon$-best-reply to $\boldsymbol{\pi}^{*-i}$ for all $i = 1, ..., N$
\end{defn}

When $\epsilon = 0$, a policy is simply called a best-reply instead of a 0-best-reply, and 0-equilibrium are simply called equilibrium.

\section{Extending the Decentralized Q-Learning Algorithm to the Continuous-Spac Setting}
\label{sec:cts_dec_alg}

In this section, we propose a decentralized multi-agent reinforcement learning algorithm for stochastic games with continuous spaces and provide our theoretical guarantees. We start by giving the relevant background on quantization of state and action spaces, and the quantized Q-learning algorithm for single-agent MDPs.

\subsection{Quantization of State and Action Spaces}
\label{ssec:quantization_state_action}

When the state space or the action space(s) in an MDP or a stochastic game are continuous, we may be interested in quantizing those spaces in order to run a Q-learning algorithm. At a high-level, we group similar states into a finite set of representative bins and learn a value function on these bins from which we can infer a nearly-optimal policy.

First, consider the quantization of the action space $\mathbb{U}$. Building on \cite{saldiAsymptoticOptimality2017}, \cite[Lemma 2.1]{karaQlearningMDPs2021} shows that any MDP with a weakly continuous transition kernel can be approximated by MDPs with finite action spaces. In view of this, we assume in this paper that the action spaces are finite, perhaps as a result of such a quantization. 

We focus on the quantization of the state space. Building on \cite{saldiAsymptoticOptimality2017}, we choose a partition of the state space $\mathbb{X}$ into $M$ disjoint sets $\{B_i\}_{i=1}^M$, such that $\cup_{i} B_i = \mathbb{X}$, and $B_i \cap B_j = \emptyset$ for $i \neq j$. In each $B_i$, we choose (any) representative state $y_i \in B_i$, and denote the quantized finite state space by $\mathbb{Y} = \{ y_1, ..., y_M \}$. We define the quantization mapping $q: \mathbb{X} \to \mathbb{Y}$ by $q(x) = y_i$ if $x \in B_i$.

We now define the \textit{finite approximation MDP}. Let $\mu^* \in \mathcal{P}(\mathbb{X})$ be any weighting measure such that $\mu^*(B_i) > 0$ for all $B_i$. Now define a normalized measure on each $B_i$ by $\hat{\mu}_{y_i}^*(A) = \frac{\mu^*(A)}{\mu^*(B_i)}$ for each $A \subset B_i$. The stage-wise cost and transition kernel on $\mathbb{Y}$ are defined as an average of those quantities in the original MDP over $B_i$ weighted by $\hat{\mu}_{y_i}^*$:
\begin{equation}
\label{eq:finite_mdp}
\begin{split}
    C^*(y_i, u) &= \int_{B_i}{c(x,u) \hat{\mu}_{y_i}^*(dx)} \\
    P^*(y_j | y_i, u) &= \int_{B_i}{\mathcal{T}(B_j | x, u) \hat{\mu}_{y_i}^*(dx)}
\end{split}
\end{equation}

We denote the optimal value function of this finite MDP by $\hat{J} : \mathbb{Y} \to \mathbb{R}$, and the optimal Q-function by $\hat{Q}^*$. This can be extended to the original state space $\mathbb{X}$ by making it constant over the quantization bins $B_i$. Slightly generalizing \cite{saldiAsymptoticOptimality2017}, \cite[Theorem 2.7]{karaQlearningMDPs2021} relates the finite approximation MDP to the original continuous-space MDP. It states that if  $\mathcal{T}$ is weakly continuous in $(x,u)$ and $c \colon \mathbb{X} \times \mathbb{U} \to \mathbb{R}_+$ is continuous and bounded, then for any compact $K \subset \mathbb{X}$ we have

\begin{equation*}
    \sup_{x_0 \in K}{\lvert \hat{J}(x_0) - J^*(x_0) \rvert} \to 0, \
    \sup_{x_0 \in K}{\lvert J(x_0, \hat{\gamma}) - J^*(x_0)\rvert} \to 0
\end{equation*}

\noindent as $\max_{i=1, ..., M-1}{\sup_{x, x' \in B_i}{\lVert x - x'\rVert}} \eqqcolon \lVert L^- \rVert_\infty \to 0$, where $\hat{\gamma}$ denotes the optimal policy of the finite-state approximate model (extended to $\mathbb{X}$ via the quantization $q$). By assumption, $\mathbb{X}$ is a Borel subset of a Euclidean space, hence it is $\sigma$-compact. Thus, for each $M$, there exists a partition $\{ B_i \}_{i=1}^M$ of $\mathbb{X}$ such that $\lVert L^- \rVert_\infty \to 0$ and $\cup_{i=1}^{M-1}{B_i} \nearrow \mathbb{X}$ as $M \to \infty$.

\subsection{Single-Agent Quantized Q-Learning}
\label{ssec:quantized_qlearning}

In view of the above results on robustness to quantization, we consider the following quantized Q-learning algorithm for continuous-space MDPs:
\begin{equation}
\label{eq:quantized_qlearning}
\begin{split}
    Q_{t+1}&(q(x), u) = \left( 1 - \alpha_t(q(x), u)\right) Q_t(q(x), u) \\
    &+ \alpha_t(q(x), u) \left( c(x, u) + \beta \min_{v \in \mathbb{U}}{Q_t(q(X_{t+1}), v)}\right)
\end{split}
\end{equation}

\noindent when $(X_t, U_t) = (x,u)$. Note that this is simply the standard Q-learning algorithm applied to the quantized states (i.e.: defined on the quantized state action pairs in $\mathbb{Y} \times \mathbb{U}$).

In \cite{karaQlearningMDPs2021}, the authors proved that under mild conditions on the transition kernel and cost function, this quantized version of the Q-learning algorithm converges to a limit which induces a near-optimal policy. In particular, the Q-learning iterations in \cref{eq:quantized_qlearning} converge to the optimal Q-function of the finite MDP model in \cref{eq:finite_mdp}. This holds if the transition kernel is weakly continuous in the state and action, the cost function is continuous and bounded, the state process is positive Harris recurrent, and $\alpha_t(y, u) = (1 + \sum_{k=1}^t{\mathbb{I}\{(y_k, u_k)=(y,u)\}})^{-1}$.  These results build on previous work in \cite{karaConvergenceFinite2021} on the convergence of finite-memory Q-learning for partially-observed Markov decision processes (POMDPs), by viewing quantization as a measurement kernel. 

\subsection{Decentralized Multi-Agent Q-Learning Algorithm for the Continuous-Space Setting}

We now propose an extension of the decentralized multi-agent Q-learning algorithm to the continuous state space setting. Our proposed algorithm closely follows the algorithm in \cite{arslanDecentralizedQlearning2017} but incorporates quantization as in \cite{karaQlearningMDPs2021}. The algorithmic description is given in \cref{alg:cts_dec_qlearning}.

Like the finite space variant, the algorithm operates on two timescales. On the finer timescale, each agent learns an estimate of the optimal value function for their current environment, composed of the underlying stochastic game and the policies of the other agents. On the coarser timescale, agents switch policies to an estimated best-reply based on their learned value function.

In the extension to the continuous space setting, each agent has a state quantizer, $q^i$ which maps the state space $\mathbb{X}$ to a finite set of representative states $\mathbb{Y}^i$. The agents' quantizers need not be the same. Hence, while all agents observe the same global state, the Q-learning algorithm they each run will in general see different quantizations of the state. Each agent's quantization induces a finite quantized \textit{baseline} policy set, $\hat{\Pi}_q^i \coloneqq \{ \hat{\pi}: \mathbb{Y}^i \to \mathbb{U}^i \}$. When choosing actions, agents follow a perturbation of their baseline policies: with probability $\rho^i$, player $i$ chooses its action uniformly at random. We denote the perturbed policy set by $\hat{\Pi}_{q, \rho}^i = \{  (1-\rho^i) \hat{\pi} + \rho^i \text{Uniform}(\mathbb{U}^i): \hat{\pi} \in \hat{\Pi}_q^i\}$. 

Each environment that an agent can see corresponds to some \textit{perturbed} quantized joint policy in the set $\boldsymbol{\hat{\Pi}}_{q, \rho} \coloneqq \times_{i \in \mathcal{N}}{\hat{\Pi}_{q, \rho}^i}$. The \textit{baseline} quantized policy set of agent $i$, $\hat{\Pi}_q^i$, is the set of policies that the agent considers when determining a response to the environment they see. It will be important to distinguish between these sets in the statement of our theoretical results.

During each exploration phase, the agents use some perturbed quantized policy in their policy space. This induces a stationary MDP environment from the perspective of each agent. Each agent runs the quantized Q-learning algorithm described in \cref{ssec:quantized_qlearning} on the MDP it sees. We will show that the quantized Q-learning algorithm converges to the optimal Q-function of the finite approximation MDP defined in \cref{ssec:quantization_state_action}, $\hat{Q}^{*i}_{\boldsymbol{\pi}_\rho^-i}$. Hence, for each `environment' $\boldsymbol{\hat{\pi}}_\rho^{-i} \in \boldsymbol{\hat{\Pi}}_{q, \rho}^{-i}$, $\hat{Q}^{*i}_{\boldsymbol{\pi}_\rho^-i}$ induces a best-reply for agent $i$ inside $\hat{\Pi}_q^i$, which is estimated using the learned $Q_t^i$. At the end of each exploration phase, each agent transitions to an estimated best-reply.

\begin{algorithm}[!htbp]
    \caption{Continuous-Space Decentralized Q-Learning Algorithm (for agent $i$)}
    \label{alg:cts_dec_qlearning}
    \setstretch{1.05}

        \textbf{Set Parameters} \\
            \codeindent $q^i$: agent $i$'s quantizer of the state space $\mathbb{X}$; $\mathbb{Y}^i \coloneqq q^i(\mathbb{X})$
            
            \codeindent $\mathbb{Q}_q^i$: some compact subset of $\mathbb{R}^{|\mathbb{Y}^i|\times|\mathbb{U}^i|}$
            
            \codeindent $\{T_k\}_{k\geq0}$: sequence of exploration phase lengths in $\mathbb{Z}_{\geq 1}$
            
            \codeindent $\rho^i \in (0,1)$: experimentation probability 

            \codeindent $\delta^i \in (0, \infty)$: tolerance for suboptimality
            
        \textbf{Initialize} $\pi_0^i \in \hat{\Pi}_q^i$, $Q_0^i \in \mathbb{Q}_q^i$ (arbitrary)
        
        \textbf{Receive} $x_0$
        
        \textbf{Iterate} $k \geq 0$ \\
        \codeindent ($k^{th}$ exploration phase)
        
        \codeindent \textbf{Iterate} $t = 1, ..., T_k$
        
        \codeindent\codeindent Quantize state: $y^i_t = q^i(x_t)$
        
        \codeindent\codeindent Choose action:
            $u_t^i = \begin{cases}
            \pi_k^i(y_t) \quad &\text{w.p.} \, 1 - \rho^i\\
            \text{any } u^i\in \mathbb{U}^i \quad &\text{w.p. } \, \rho^i \\
            \end{cases}$

        \codeindent\codeindent \textbf{Receive} $c^i(x_t, \boldsymbol{u}_t)$

        \codeindent\codeindent \textbf{Receive} $x_{t+1} \sim \mathcal{T}\left( \cdot \;\middle|\; x_t, u_t^i, \boldsymbol{u}_t^{-i}\right)$,\\ 
        \codeindent\codeindent Quantize: $y_{t+1}^i = q^i(x_{t+1})$

        
        \codeindent\codeindent $\alpha_t^i(y_t^i, u_t^i) = \left(1 + \sum_{s=t_k}^{t}{\mathbb{I}\left\{(y_s^i, u_s^i) = (y_t^i, u_t^i)\right\}} \right)^{-1}$

        \begin{fleqn}[\widthof{\codeindent\codeindent}]
        \begin{equation*}
        \begin{aligned}
            &\begin{aligned}
                Q_{t+1}^i&(y_t^i, u_t^i) = \left(1 - \alpha_t^i(y_t^i, u_t^i)\right) Q_t^i(y_t^i, u_t^i) \\
                &+ \alpha_t^i(y_t^i, u_t^i) \left(c^i(x_t, \boldsymbol{u}_t) + \beta \min_{v^i \in \mathbb{U}^i}{Q_t^i(y_{t+1}^i, v^i)} \right)
            \end{aligned} \\
           &\begin{aligned} 
                Q_{t+1}^i(y^i, u^i) = Q_t^i(y^i, u^i), \quad \forall (y^i,u^i) \neq (y_t^i, u_t^i)
            \end{aligned}
        \end{aligned}
        \end{equation*}
        \end{fleqn}
        
        \codeindent \textbf{End}
        
        \begin{fleqn}[\widthof{\codeindent}]
            \begin{align*}
            \begin{split}
                \widehat{\text{BR}}_\delta^i(Q_t^i) = \left\{ \hat{\gamma} \in \hat{\Pi}_q^i \colon Q_t^i(y, \hat{\gamma}(y)) \leq \min_{u \in \mathbb{U}^i}{Q_t^i(y, u)} + \delta^i, \, \forall y \in \mathbb{Y}_q^i \right\}
            \end{split}
            \end{align*}
        \end{fleqn}

        \codeindent \textbf{If} $\pi_k^i \in \widehat{\text{BR}}_\delta^i(Q_t^i)$\\
        \codeindent\codeindent $\pi_{k+1}^i = \pi_{k}^i$

        \codeindent \textbf{Else} \\
        \codeindent\codeindent Choose $\pi_{k+1}^i \in \widehat{\text{BR}}_\delta^i(Q_t^i)$

        \codeindent \textbf{End}

        \codeindent Reset $Q_t^i$ to any $Q^i \in \mathbb{Q}_q^i$ (e.g.: $Q_t^i = 0$)
    
    \textbf{End}

\end{algorithm}

\begin{assumption}
    \label{ass:cts_decqlearning}
    \hphantom{~}
    
    \begin{enumerate}
        \item $\mathcal{T}$ is weakly continuous in $(x,\boldsymbol{u})$

        \item $c^i \colon \mathbb{X} \times \boldsymbol{U} \to \mathbb{R}_+$ is continuous and bounded for every $i$.

        \item Under every stationary joint policy $\boldsymbol{\gamma} \in \boldsymbol{\hat{\Pi}}_{q,\rho}$, the state process $\{X_t\}_t$ is positive Harris recurrent, and thus admits a unique invariant measure, $\mu_{\boldsymbol{\gamma}}^*$
        
        \item For each stationary joint policy $\boldsymbol{\gamma} \in \boldsymbol{\hat{\Pi}}_{q, \rho}$, let $\mathbb{P}^{\boldsymbol{\gamma}}$ denote the probability measure associated with $\boldsymbol{\gamma}$. For every agent $i \in \mathcal{N}$, every observation-action pair $(y^i, u^i) \in \mathbb{Y}^i \times \mathbb{U}^i$, we suppose that $(y^i, u^i)$ is visited infinitely often $\mathbb{P}^{\boldsymbol{\gamma}}$-almost surely.
    \end{enumerate}
\end{assumption}

\begin{theorem}
\label{thm:cts_transition_near_br}
\hphantom{~}
Suppose all players use Algorithm~\ref{alg:cts_dec_qlearning} to select their actions. For any $\epsilon > 0$, there exists $\tilde{T}$ such that $T_k \geq \tilde{T}$ implies
\begin{equation}
    \mathbb{P}\left[ \left\| Q_{T_k}^i - \hat{Q}^{*i}_{\boldsymbol{\pi}_{k, \rho}^{-i}} \right\|_\infty < \epsilon  \right] \geq 1 - \epsilon, \quad \forall k \geq 0,
\end{equation}

where $\boldsymbol{\pi}_k$ is the baseline joint policy during the $k^{\rm th}$ exploration phase and $\boldsymbol{\pi}_{k, \rho}$ is the perturbation of $\boldsymbol{\pi}_k$ that is used for action selection.

Furthermore, for any $\eta > 0$, there exists a fine-enough quantization $q^i$ such that the greedy policy with respect to $\hat{Q}^{*i}_{\boldsymbol{\pi}_\rho^{-i}}$ is $\eta$-optimal in the MDP environment where the other agents act according to $\boldsymbol{\pi}_\rho^{-i}$.
\end{theorem}

\begin{proof}
Suppose each agent $i$ uses the baseline policy $\pi^i \in \hat{\Pi}^i$. Denote the perturbed policy that agent $i$ follows during the exploration phase by $\bar{\pi}^i \sim \rho^i \text{Unif}(\mathbb{U}^i) + (1 - \rho^i) \pi^i$, a mixture of the deterministic policy $\pi^i$ and i.i.d. uniformly random actions.

Consider agent $i$'s Q-learning updates. First, we show that agent $i$ sees a stationary MDP. At each time step, agent $i$ incurs the random cost $c^i(x, u^i, \boldsymbol{u}^{-i}), \boldsymbol{u}^{-i} \sim \bar{\boldsymbol{\pi}}^{-i}(x)$. Note that this is a random function of $x, u$ and that the random component is independent and identically distributed. This is equivalent to the MDP with deterministic cost function $\tilde{c}^i(x, u^i) \coloneqq \mathbb{E}_{\boldsymbol{u}^{-i} \sim \bar{\boldsymbol{\pi}}^{-i}(x)}\left[c(x, u^i, \boldsymbol{u}^{-i})\right]$. (As an aside, note that $\tilde{c}^i(x, u^i)$ converges to $c^i(x, u^i, \boldsymbol{\pi}^{-i}(x))$ as $\rho^j \to 0$ for $j \neq i$.)

Similarly, the distribution over next states seen by agent $i$ when at state $x$ and they take action $u^i$ is given by the mixture $\mathbb{E}_{\boldsymbol{u}^{-i} \sim \bar{\boldsymbol{\pi}}^{-i}(x)}\left[\mathcal{T}(\cdot | x, u^i, \boldsymbol{u}^{-i}\right]$. (Note similarly that this converges weakly to the transition kernel $\mathcal{T}(\cdot | x, u^i, \boldsymbol{\pi}^{-i}(x))$ as $\rho^j \to 0$ for all $j \neq i$.)

Thus, agent $i$ sees a stationary MDP (with a continuous state space), which we denote by $M_{\boldsymbol{\bar{\pi}}}^i$. Furthermore, since $\mathcal{T}(\cdot| x, \boldsymbol{u})$ is weakly continuous by assumption, the transition kernel of $M_{\boldsymbol{\bar{\pi}}}^i$ is weakly continuous in $x, u^i$. Also, since $c^i$ is continuous and bounded in $\mathbb{X}\times \mathbb{U}$, $\tilde{c}$ is continuous and bounded in $\mathbb{X}\times\mathbb{U}^i$ (boundedness is immediate, and continuity follows from the dominated convergence theorem: $\lim_{k \to \infty}\tilde{c}(x_k, u_k^i) \coloneqq \lim_{k \to \infty}\mathbb{E}\left[c(x_k, u_k^i, \boldsymbol{u}^{-i})\right] = \mathbb{E}\left[\lim_{k \to \infty} c(x_k, u_k^i, \boldsymbol{u}^{-i})\right] = \tilde{c}(x, u^i)$)

Invoking \cite[Theorem 3.2, Corollary 3.3]{karaQlearningMDPs2021}, we have that the quantized Q-learning algorithm for agent $i$ converges almost surely to $\hat{Q}^{*i}_{\boldsymbol{\bar{\pi}}^{-i}}$. Furthermore, any greedy policy with respect to $\hat{Q}^{*i}_{\boldsymbol{\bar{\pi}}^{-i}}$ is nearly optimal. More precisely, if $\hat{\gamma}$ is such that $\hat{Q}^{*i}_{\boldsymbol{\bar{\pi}}^{-i}}(x, \hat{\gamma}(x)) = \min_{u^i \in \mathbb{U}^i}{\hat{Q}^{*i}_{\boldsymbol{\bar{\pi}}^{-i}}(x, u^i)}$, then for any compact $K \subset \mathbb{X}$, we have $\sup_{x_0 \in K}{\lvert J(x_0, \hat{\gamma}) - J^*(x_0)\rvert} \to 0$ as $\lVert L^{-} \rVert \to \infty$ where $\lVert L^-\lVert_\infty \coloneqq \max_{j\in \{1, ..., M-1\}}{\sup_{x, x' \in B_j^i}{\lVert x - x' \rVert}}$ and $\{B_j^i\}_j$ are the quantization bins of agent $i$'s quantizer $q^i$.
\end{proof}

The Q-factors that each agent learns, and hence their policy updates, correspond to the environment they see during that exploration phase, which includes the (perturbed) policies that the other agents use. Thus, each agent is continually responding to their last-seen environment in a nearly-optimal way. This is a sensible goal for each agent to have, independently of all other agents. How close an agent's policy update is to optimality at the end of each exploration phase is determined by the `resolution' of their quantizer. We can refer to the limiting state-action value function $\hat{Q}^{*i}_{\pi^{-i}}$ as agent $i$'s `subjective' value function since it is induced by their subjective observations via their quantizer.

In some applications, the goal might be to guarantee that the algorithm globally converges to an equilibrium. As the analysis in \cref{sec:policy_updating_dynamics} shows, under certain further assumptions and adjustments to the algorithm, the algorithm converges to an equilibrium. 

\section{Policy-Updating Dynamics and Convergence to Equilibrium}
\label{sec:policy_updating_dynamics}


In the previous section, our theorem was concerned with the algorithm's stage-wise behavior from the perspective of each agent. In particular, the convergence of each agent's Q-learning process and the near-optimality of the policy updates that those limiting Q-values induce. However, one may also be interested in understanding the algorithm's global dynamics over the joint policy space, and in particular, knowing whether the algorithm converges to some type of equilibrium.

In this section we derive a characterization of our algorithms' policy updating dynamics as a Markov chain over the joint policy space (when the exploration phases are long enough so that Q-learning converges). This characterization will apply to a general class of best-reply-based algorithms including the original finite-space decentralized Q-learning algorithm in \cite{arslanDecentralizedQlearning2017} and the continuous-space algorithm (\cref{alg:cts_dec_qlearning}). Furthermore, under certain additional assumptions, we derive a closed-form expression for the probability of convergence to each equilibrium. 


In what follows, we characterize the global policy-updating dynamics of the algorithms as a Markov chain, derive a closed-form expression for the probabilities  of convergence to each equilibrium, and close by reflecting on best-reply-based MARL algorithms when the goal, in the case of the cooperative team setting, is achieving global optimality.

\subsection{Global Policy-Updating Dynamics Modeled as a Markov Chain}

The approach we take is to first analyze an idealized update process, then show that when the exploration phases are sufficiently long, the results derived for the idealized process hold with arbitrarily high probability and precision in the true stochastic algorithms.

Consider a general class of algorithms following the idealized policy updates in \cref{alg:idealized_updating_dynamics}. When an agent is best-responding, it does not switch its policy. When not best-responding, the agent randomly chooses a policy to switch to according to a distribution $\psi$ which can depend on the best-reply set. In the finite-space algorithm from \cite{arslanDecentralizedQlearning2017}, the agent stays at its current policy with probability $\lambda^i$ (`inertia' parameter) and uniformly switches to a policy in its best-reply set otherwise. In \cref{alg:cts_dec_qlearning}, the agent always switches to a best-response (since their goal is to be individually best-responding rather than to converge to equilibrium necessarily). We consider $\psi$ to be general in this analysis. 


Note that in this section we will use $\text{BR}$ to denote limiting best-reply sets and $\boldsymbol{\Pi}$ to denote the joint policy space (those quantities will depend on the agents' quantizers in the case of the quantized continuous-space algorithm).

\begin{algorithm}[H]

    \caption{Generalized Idealized Policy-Updating Process (for agent $i$)}
    \label{alg:idealized_updating_dynamics}
        \textbf{Set Parameters}\\
        \codeindent $\psi^i(\cdot; \cdot)$ some prob. dist. over $\Pi^i$ given a best-reply set
        
        \textbf{Initialize} $\pi_0^i \in \Pi^i$ (arbitrarily)
        
        \textbf{Iterate} $k \geq 0$
        
        \codeindent \textbf{If} $\pi_k^i \in \text{BR}^i(\boldsymbol{\pi}^{-i})$ \\
        \codeindent\codeindent $\pi_{k+1}^i = \pi_k^i$
        
        \codeindent \textbf{Else} \\
        \codeindent\codeindent $\pi_{k+1}^i = \gamma^i \in \Pi^i \qquad \text{w.p. } \psi^i(\gamma^i; \text{BR}^i(\boldsymbol{\pi}^{-i}))$
                                 
        \codeindent \textbf{End}
\end{algorithm}

We note that this is a stationary Markov chain in the space of joint policies, $\boldsymbol{\Pi}$. That is, when the agents are at joint policy $\boldsymbol{\pi_k}$, their transition to the next joint policy $\boldsymbol{\pi}_{k+1}$ is independent of all preceding policies, and is given by the dynamics in \cref{alg:idealized_updating_dynamics}. Furthermore, each agent updates its policy $\pi^i$ in response to $\boldsymbol{\pi}^{-i}$ independently of the other agents' updates.  We next derive the transition kernel of this Markov chain. To make the dependence explicit, we denote this transition matrix by $P(\boldsymbol{\psi}; \text{BR}_G)$, where $\text{BR}_G$ is the best-reply graph induced by the limiting Q functions. 

For ease of computation, we first obtain an expression for the probability of transitioning from a given joint policy to a joint policy matching at agent $i$:
\begin{equation}
\label{eq:agent_tr_prob}
\begin{split}
    P_{\boldsymbol{\pi}, \pi'^i}^i &\coloneqq \mathbb{P} \left[ \pi_{k+1}^i = \pi'^i | \boldsymbol{\pi}_k = \boldsymbol{\pi} \right] \\
    &= \begin{cases}
    1 & \text{if }  \pi^i \in \text{BR}^i(\boldsymbol{\pi}^{-i}), \, \pi'^i= \pi^i, \\
    \psi^i(\pi'^i; \text{BR}^i(\boldsymbol{\pi}^{-i})), &\text{otherwise}
    \end{cases}
\end{split}
\end{equation}

Now, for any pair of joint policies $(\boldsymbol{\pi}, \boldsymbol{\pi}')$, we get that the probability of transitioning from $\boldsymbol{\pi}$ to $\boldsymbol{\pi}'$ is:

\begin{equation}
    \label{eq:joint_tr_prob}
    P_{\boldsymbol{\pi}, \boldsymbol{\pi}'} \coloneqq \mathbb{P} \left[ \boldsymbol{\pi}_{k+1} = \boldsymbol{\pi}' | \boldsymbol{\pi}_k = \boldsymbol{\pi} \right] = \times_{i=1}^{N} {P_{\boldsymbol{\pi}, \pi'^i}^i}
\end{equation}

\subsection{Convergence to Equilibrium: Closed-form Probabilistic Characterization}

Since our goal in this section is to characterize convergence to equilibrium, we require that $P(\boldsymbol{\psi}; \text{BR}_G)$ is the transition matrix of an absorbing Markov chain. In \cite{arslanDecentralizedQlearning2017}, the authors provide a sufficient condition on the class of stochastic games which guarantees convergence to an equilibrium. This condition is termed `weak acyclicity' and it assumes that there exists a sequence of policies from any joint policy to some equilibrium policy in which any two consecutive joint policies differ in only one agent and that agent's new policy is a best-response. In view of the characterization above, this condition clearly implies the constructed Markov chain is absorbing. However, this is not a necessary condition since, depending on $\boldsymbol{\psi}$, there exists other finite-paths to equilibrium policies. For a fixed game with a particular best-reply graph, choosing $\boldsymbol{\psi}$ which gives positive probability to each policy in $\Pi^i$ induces an absorbing Markov chain for the maximal set of stochastic games. For example, $\boldsymbol{\psi}$ could be such that the agent transitions to a best-reply w.p. $1 - \epsilon$ but explores their policy space uniformly w.p. $\epsilon$. For the following results, we assume that $\psi$ and $\text{BR}_G$ are such that $P(\boldsymbol{\psi}; \text{BR}_G)$ is absorbing. This assumption is in the same spirit as the assumption of weak acyclicity in \cite{arslanDecentralizedQlearning2017}, but is weaker and more general. In the next section, we will consider the special case of \cref{alg:cts_dec_qlearning}. Using standard results on absorbing Markov chains, we can derive a closed-form expression for the probability of convergence to any equilibrium in the stochastic game. Under \cref{alg:idealized_updating_dynamics}, an equilibrium corresponds to an absorbing state in Markov chain terminology. From \cref{eq:agent_tr_prob}, we see that this occurs precisely when $\pi^i \in \text{BR}^i(\boldsymbol{\pi}^{-i})$ for all $i \in \{1, ..., N\}$. All other joint policies are ``transient states''. Let $t$ be the number of transient states and $r$ the number of absorbing states. Thus, if we index the joint policy space appropriately, we obtain a transition matrix in absorbing canonical form:

\begin{equation*}
    P = \begin{bmatrix}
        Q & R\\ 
        \mathbf{0} & I_r
        \end{bmatrix}
\end{equation*}

Where $Q$ is the matrix of transition probabilities between transient states, $R$ the matrix of transition probabilities from transient states to absorbing states, and $I_r$ is the identity matrix corresponding to the fact that $P_{\boldsymbol{\pi}, \boldsymbol{\pi}} = 1$ for an absorbing state.

\begin{remark}
    In this section, what we mean by `equilibrium' is an equilibrium according to $\text{BR}_G$. If the limiting Q-values correspond to the true state-action value function, this is a true equilibrium; if they differ, then it is a `subjective' equilibrium.
\end{remark}

\begin{prop}
    \label{prop:abs_probs}
    Let a stochastic game $G$ be given, inducing a best-reply graph $\text{BR}_G$. Define $A(\boldsymbol{\psi}, \text{BR}_G)$ to be the vector giving the probability of converging to each joint policy in $\boldsymbol{\Pi}$. Suppose the transition matrix over $\boldsymbol{\Pi}$ is given in absorbing canonical form according to the derivation above. Then, the vector of convergence probabilities is given by
    \begin{equation}
        A(\boldsymbol{\psi}, \text{BR}_G) = A_0 
            \begin{bmatrix}
            \mathbf{0} & (I_t - Q)^{-1} R \\ 
            \mathbf{0} & I_r
            \end{bmatrix},
    \end{equation}
    
    \noindent where $A_0$ is a vector giving the initial distribution over the joint policy space (e.g.: uniform if each agent chooses their initially policy uniformly at random).
\end{prop}
\begin{proof}
    Let $(P_\infty)_{i,j}$ be the matrix giving the probability of converging to joint policy $j$ if starting at joint policy $i$. By \cite[Theorem 3.2.7]{kemenyFiniteMarkov1983}, for all transient initial states the probabilities are given by $(I_t - Q)^{-1} R$. For all absorbing initial states, we remain in the same state, which gives probabilities $I_r$. Finally, by \cite[Theorem 3.1.1]{kemenyFiniteMarkov1983}, the process cannot converge to a transient state, so those probabilities are $0$. Thus,
    
    \begin{equation*}
        P_\infty = \begin{bmatrix}
                    \mathbf{0} & (I_t - Q)^{-1} R\\ 
                    \mathbf{0} & I_r
                    \end{bmatrix}.
    \end{equation*}
    
    Finally, to get the absorption probabilities conditioned on the initial distribution over joint policies, we multiply by $A_0$.
\end{proof}

Thus, for the idealized update process of this general class of best-reply-based decentralized algorithms, we derived an explicit expression of the probability of converging to each equilibrium in $\boldsymbol{\Pi}$ (as a function of the algorithm parameters $\{\psi^i\}_{i \in \mathcal{N}}$ and the game's best-reply graph $\text{BR}_G$). Note that in addition to the probability of convergence to each equilibrium, absorbing Markov chain theory also tells us the expected number of policy updates required until convergence. In particular, the expected number of steps until convergence when starting at non-equilibrium policy $i$ is $\left((I_t - Q)^{-1} \boldsymbol{1}\right)_i$.

As noted, the finite-space and continuous-space decentralized Q-learning algorithms are stochastic approximations of this idealized process, since the $Q_t^i$ factors used to estimate their best-reply sets are noisy estimates of the limiting Q-values. Next, we show that for any algorithm approximating \cref{alg:idealized_updating_dynamics} with Q-iterations which converge, our results hold approximately.


Consider a stochastic algorithm following \cref{alg:idealized_updating_dynamics} with a particular form for $\{\psi^i\}_{i\in \mathcal{N}}$ but where the best-reply sets of each agent are estimated with learned Q-values. Let a stochastic dynamic game be given which induces a transition kernel $P(\boldsymbol{\psi}; \text{BR}_G)$ for the idealized policy-updating algorithm corresponding to an absorbing Markov chain in the joint policy space. Denote by $P_k$ the stochastic realization of the transition kernel at exploration phase $k$ for the algorithm in consideration. Note that $P_k$ is stochastic and determined by the approximate best-reply sets estimated by the learned Q-factors $Q_{t}^i$ at the end of each exploration phase. Denote by $P_{0:K}$ the matrix of terminating probabilities of the stochastic algorithm after $K$ exploration phases. That is, $(P_{0:K})_{i,j}$ is the probability the algorithm is at joint policy $j$ after $K$ exploration phases given that the initial joint policy was $i$, $P_{0:K} = \times_{k=0}^{K}{P_k}$. $P_{0:K}$ is also stochastic and depends on the convergence of Q-learning for the first $K$ exploration phases.

\begin{prop}
    \label{thm:approx}
    Consider a stochastic algorithm as described above approximating \cref{alg:idealized_updating_dynamics} with corresponding $\boldsymbol{\psi} = \{\psi^i\}_{i \in \mathcal{N}}$. Let the best-reply sets induced by the limiting Q-values of the algorithm be given by $\text{BR}_G$. Recall that the corresponding idealized best-reply process has a transition kernel $P(\boldsymbol{\psi}; \text{BR}_G)$ and absorption matrix $P_\infty$ (as defined above). Suppose that $\boldsymbol{\psi}$ and $\text{BR}_G$ are such that $P(\boldsymbol{\psi}; \text{BR}_G)$ is absorbing. Suppose that the stochastic algorithm's Q-iterations converge at every long-enough exploration phase (for this, the assumptions in \cite[Theorem 1]{arslanDecentralizedQlearning2017} or \cref{thm:cts_transition_near_br} are sufficient for their respective algorithms). Then, for any $\epsilon > 0$, there exists $\bar{K}$, such that for any given $K > \bar{K},\ K < \infty$, there exists $\bar{T}_K$ such that if $T_k > \bar{T}_K$ for $k=0,...,K$, then
    \begin{equation*}
        \mathbb{P}\Big[ \left|P_{0:K} - P_\infty \right|_\infty < \epsilon \Big] \geq 1 - \epsilon
    \end{equation*}
\end{prop}


\begin{proof}
    Since the idealized best-reply process is an absorbing Markov chain, $\lim_{k \to \infty}{P^k}$ converges. So, for any $\epsilon > 0$, there exists $\bar{K}$ such that
    \begin{equation}
        \label{eq:approx_P_inf}
        \left| P^K  - P_\infty \right|_\infty < \epsilon, \quad \text{for all } K > \bar{K}.
    \end{equation}
    
    Note that the event 
    \begin{equation*}
        E_k \coloneqq \left\{ \omega \colon \left| Q_{\pi_k^{-i}}^i - Q_{\bar{\pi}_k^{-i}}^i \right|_\infty \leq \frac{1}{2} \min\{\delta^i, \bar{\delta} - \delta^i\}, \, \forall i \right\}    
    \end{equation*}
    implies that $\Pi_k^i = \Pi_{\pi_k^{-i}}^i$ for all $i$. That is, the estimated best-reply set of each agent is the same as the true best-reply set. If the best-reply sets are equal at exploration phase $k$, then the transition kernel for the Q-learning algorithm is the same as that of the idealized best-reply process for that exploration phase ($P_k = P$).
    
    \cite[Lemma 4]{arslanDecentralizedQlearning2017} states that if $T_k$ is large enough, then $E_k$ can have probability arbitrarily close to one. Since $K$ is finite, by \cite[Lemma 4]{arslanDecentralizedQlearning2017}, for $\epsilon > 0$, there exists $\bar{T}_K$ such that if $T_k > \bar{T}_K$ for all $k = 0, ..., K$ and $\delta^i \in (0, \bar{\delta}), \rho^i \in (0, \bar{\rho})$, for all $i$, then
    \begin{equation}
        \mathbb{P}\left[ E_0, ..., E_K \right] \geq 1 - \epsilon.
    \end{equation}
    Note that $\bar{T}_K$ depends on the specific choice of the number of exploration phases $K$, and not just the required lower bound $\bar{K}$.
    
    Then, with $T_k > \bar{T}_K$ for all $k = 1, ..., K$, we have that 
    \begin{equation*}
    \begin{split}
        Pr[P_{0:K} = P^K] &\geq Pr[P_i = P \ \text{for all } i = 0, ..., K] \\
        &\geq Pr[E_0, ..., E_k]  \\
        &\geq 1 - \epsilon,
    \end{split}
    \end{equation*}
    
    \noindent since $E_0, ..., E_k \implies P_i = P \ \text{for all } i = 0, ..., K \implies P_{0:K} = P^K$.
    
    By \cref{eq:approx_P_inf}, we have that for $K > \bar{K}$, $\left| P^K  - P_\infty \right|_\infty < \epsilon$. This directly implies that 
    
    \begin{equation*}
        \mathbb{P}\Big[ \left|P_{0:K} - P_\infty \right|_\infty < \epsilon \Big] \geq 1 - \epsilon.
    \end{equation*}
\end{proof}

\begin{remark}
    Note that the finiteness of the number of exploration phases $K$ is crucial. Thus, the result is not uniform over $K > \bar{K}$. Without finiteness, the algorithm still converges to equilibrium, but the absorption probabilities may deviate from this closed-form expression.
\end{remark}

\begin{cor}
    \label{cor:approx_abs}
    In the above setup, suppose that the initial policies of each agent are chosen randomly according to $A_0$. Let $A^s_k$ be the vector giving the probability that the decentralized Q-learning algorithm is at each policy after $k$ exploration phases, $(A^s_k)_{i} \coloneqq \mathbb{P}[\boldsymbol{\pi}_k = i \, | \, \boldsymbol{\pi}_0 \sim A_0]$.
    
    Suppose the assumptions in \cref{prop:abs_probs} hold. Then for any $\epsilon > 0$, there exists $\bar{T}_K,\ \bar{K}$ such that if $T_k \geq \bar{T}_K,\ K \geq \bar{K}$, then:
    \begin{equation*}
        \mathbb{P}\Big[ \left| A^s_K - A(\boldsymbol{\psi}, \text{BR}_G) \right|_\infty < \epsilon \Big] \geq 1 - \epsilon
    \end{equation*}
    
    Where $A(\boldsymbol{\psi}, \text{BR}_G)$ is the vector of absorption probabilities for the idealized process (given random initialization according to $A_0$) as defined in \cref{prop:abs_probs}.
\end{cor}
\begin{proof}
    Recall that $A^s_K = U_{|\boldsymbol{\Pi}|}P_{0:K}$ and $A(\boldsymbol{\lambda}, \text{BR}_G) = U_{|\boldsymbol{\Pi}|} P_\infty$, as per \cref{prop:abs_probs}. 
    
    Then, by \cref{thm:approx} for any $\epsilon > 0$, there exists $\bar{T}_K, \bar{K}$ such that $T_k \geq \bar{T}_K, K \geq \bar{K}$ implies $\mathbb{P}\Big[ \left| A^s_K - A(G, \boldsymbol{\lambda}) \right|_\infty < \epsilon \Big] \geq 1 - \epsilon$.
\end{proof}
\begin{remark}
    Note that there are two sources of stochasticity in the above analysis. The first is the random policy initialization, and the second is the convergence of Q-learning within the required tolerance at each exploration phase.
\end{remark}

The above results make two main assumptions: that the noisy estimates of the limiting Q-values converge almost surely, and that $P(\boldsymbol{\psi}; \text{BR}_G)$ is the transition kernel of an absorbing Markov chain. In the finite-space algorithm \cite{arslanDecentralizedQlearning2017} the first assumption is satisfied via the convergence of Q-learning for finite-MDPs, and the second assumption clearly holds for the class of weakly acyclic stochastic games which the authors consider. In the case of \cref{alg:cts_dec_qlearning} the first assumption follows from results on convergence of quantized Q-learning in continuous-space MDPs and is stated in \cref{thm:cts_transition_near_br}. What remains, as will be studied in the following subsection, is the question of whether the second assumption holds.  

{\bf Shortcomings of best-reply-based algorithms.} In view of the analysis above, one shortcoming of the best-reply-based policy-updating dynamics is that the policy updates and the exploration of the policy space is determined only by instantaneous information in the best-reply graph. That is, agents following these algorithms switch to policies which are best-responses to their current environment, with no explicit record of prior history. Algorithms using such policy updates are able to guarantee convergence to player-by-player equilibria (local optimality), but they are unable to distinguish between global and local optimality in the cooperative team setting. In fact, the closed-form probabilities of convergence that we derive above show that, in the worst case, convergence to global optimality can occur with arbitrarily low probability. That is, there exists team problems and choices of parameters where convergence to a globally optimal solution only occurs if the random initialization of policies happens to be a global optimum. This behavior is particularly undesirable in team settings where the goal is finding a team-optimal joint policy. It can be shown that, in general, a player-by-player optimal solution can be arbitrarily worse than a team-optimal \cite{yongacogluDecentralizedLearning2022}. 


\section{Convergence to Equilibrium} \label{sec:convergence_to_eq}

As noted earlier, whether such MARL algorithms converge to a notion of equilibrium depends on the subjective state-value functions seen by each agent at each possible environment. We first answer the following question: if \cref{alg:cts_dec_qlearning} (perhaps with $\{\psi^i\}$ adjusted), what does it converge to? That is, what is a `subjective equilibrium' according to this algorithm?

\begin{prop}
    If the algorithm converges to a subjective equilibrium $\boldsymbol{\pi}$ for some game (i.e.: agents stop updating their policies), then the joint policy it converges to is an $\epsilon$-equilibrium in $\boldsymbol{\Pi}$, where $\epsilon \to 0$ as the quantization gets finer.
\end{prop}
\begin{proof}
    Suppose that for some stochastic game, the algorithm converges to a subjective equilibrium. Then each agent is best-responding with respect to their subjective value function, and hence by \cref{thm:cts_transition_near_br}, agent $i$ is within $\eta(q^i)$ of an optimal response to the environment, where $\eta(q^i) \to 0$ as $q^i$ gets finer (i.e.: $\lVert L^- \rVert \to 0$). Hence, $\boldsymbol{\pi}$ is a $\max_{i \in \mathcal{N}}{\eta(q^i)}$-equilibrium in $\boldsymbol{\Pi}$.
\end{proof}

Thus, if the algorithm converges, then the joint policy it has converged to is a near-equilibrium. The other side of the question of convergence is such: when will the algorithm converge? By our analysis in the previous section, this is equivalent to asking if there exists some path with positive probability in $P(\boldsymbol{\psi}; \text{BR}_G)$ which terminates at a subjective equilibrium for each starting joint policy. If we choose $\psi^i$ as some distribution which puts positive probability at every policy in $\hat{\Pi}^i$, then the set of paths with positive probability allowable by $\text{BR}_G$ is maximized. The condition is then met if, for example, the underlying true best-reply graph has such a path, every best-reply is unique, and the quantization is fine-enough. More generally, this is closely related to the `satisficing paths' property as defined in \cite{yongacogluSatisficingPaths2021}; we refer the reader to that reference for a treatment of how to establish such a property in certain classes of stochastic games.

\subsection{Games with Convergence Guarantees}

In the following, we discuss examples of games in which self-play of our algorithm is convergent. We begin by describing static stochastic teams with independent measurements, where the algorithm can be used without modification, and then we describe symmetric $N$-player stochastic games, for which a modified version of the algorithm drives policies to $\epsilon$-equilibrium. 

\subsection*{Decentralized stochastic team problems with independent measurements.} 

An important class of continuous-space stochastic games for which our algorithm is guaranteed to converge is static stochastic teams with independent measurements, viewed as repeated games where the state is independent and identically distributed (i.e.: the transition kernel is not a function of the current state or joint action). The independent measurements condition is a mild one, as one can, under absolute continuity conditions of measurement kernels, via a change of measures argument building on Witsenhausen \cite{wit88}, reduce a dynamic team problem to one which is static and with independent measurements. In this setting, the agents affect each other only via their cost functions since there are no state dynamics. Since this is a static team problem, the existence of `satisficing paths' to an equilibrium holds whenever an equilibrium exists (via a sequential update argument leading to a monotonically non-increasing expected cost sequence converging to a limit). \cite[Theorem 5.2]{YukselWitsenStandardArXiv} guarantees the existence of an optimal solution for such static problems when the (reduced) cost is continuous in the action variables and the action spaces are compact, as well as the existence of an $\epsilon$-optimal solution among quantized policies building on \cite{saldiyuksellinder2017finiteTeam}. Convergence of Q-learning updates to near local equilibria follows from the independence of measurements under static reduction.

\subsection*{Symmetric $N$-Player Games}

The algorithm presented in this paper prescribed switching one's policy to a (subjectively) greedy policy whenever the agent did not estimate itself to be best-replying. Under this prescription, guarantees of convergence to equilibrium depend on the existence of suitable paths, as described earlier.

Two natural modifications of the algorithm as follows: (1) replace the 0-best-reply condition in the policy update with an $\epsilon$-best-reply condition, so that the agent opts not to change policies whenever it is already (subjectively) $\epsilon$-best-replying to its environment; (2) when not $\epsilon$-best-replying to one's environment, replace the best-reply policy update with random search over the entire stationary randomized policy set. 

The modifications described above were proposed in \cite{yongacogluSatisficingPaths2021} for stochastic games with finitely many states, and it was shown that the algorithm presented there drives policies to $\epsilon$-equilibrium in \emph{symmetric} $N$-player stochastic games. The analysis there depended on the existence of so-called $\epsilon$-satisficing paths to $\epsilon$-equilibrium from any initial policy, and leveraged symmetry to explicitly establish the existence of such paths in symmetric games with finitely many states.

When generalizing the main algorithm of \cite{yongacogluSatisficingPaths2021} to symmetric games with general state spaces, the translation of convergence guarantees is not automatic: indeed, $\epsilon$-equilibrium need not exist when the state space is general \cite{levy2013discounted,levy2015corrigendum}. However, when an $\epsilon$-equilibrium, or even a subjective $\epsilon$-equilibrium, can be guaranteed to exist in a general symmetric game, the results of \cite{yongacogluSatisficingPaths2021} can be invoked, and that algorithm will drive play to $\epsilon$-equilibrium, provided all agents employ the same quantizer. We leave this extension as an interesting line for future research.

\section{Simulation Study}
\label{sec:simulation_study}

We now briefly present simulation results of \cref{alg:cts_dec_qlearning} on a simple 2-player stochastic team. Consider the stochastic team where $\mathbb{X} = [0,1],\ \mathbb{U}^1 = \mathbb{U}^2 = \{-1, 1\},\ c^1(x,\boldsymbol{u}) = c^2(x,\boldsymbol{u}) = -x$, and the state dynamics are given by $x_{t+1} = \left[x_t + 0.1 u_1 u_2 + 0.1 \xi \nu\right]_0^1$, where $\nu \sim \text{Uniform}(-1, 1)$, $\xi \sim \text{Bernoulli}(0.1)$, and $[\cdot]_0^1$ denotes clipping to $[0,1]$. Thus, states with lower values have higher costs, and the team would want to head towards $1$. The transition kernel is such that the state jumps forward by $0.1$ if the agents cooperate (play the same action) and backwards by $0.1$ otherwise, with some uniform random noise being added with probability $0.1$. Each agent uses a uniform random quantizer onto the set $\mathbb{Y}^1 = \mathbb{Y}^2 = \{0, 0.25, 0.5, 0.75, 1.0\} = \mathbb{Y}$, via $q^1(x) = q^2(x) = \argmin_{y \in \mathbb{Y}}{|y - x|}$.

We ran a version of \cref{alg:cts_dec_qlearning} with the modification that the agents choose not to update their policy (`inertia') with probability $0.25$ and $0.75$, respectively. We ran 50 trials for each exploration phase length $T$ in $\{10^2, 10^3, 10^4, 10^5\}$. For each trial we kept track of the proportion of time that a team-optimal policy was used across the first 10 exploration phases (i.e. a joint policy in which the agents cooperate at every state value). For each trial, the initial joint policy is chosen such that the agents are anti-cooperating at every state, and the initial state is chosen uniformly at random in $[0,1]$, The values are averaged across trials and shown in \Cref{fig:sim_results}. Code for implementing \cref{alg:cts_dec_qlearning} and reproducing the simulation results can be found at \url{https://github.com/Awni00/decentralized-MARL-general-cts-spaces}.

 
 \begin{figure}
     \centering
     \includegraphics[width=0.5\textwidth]{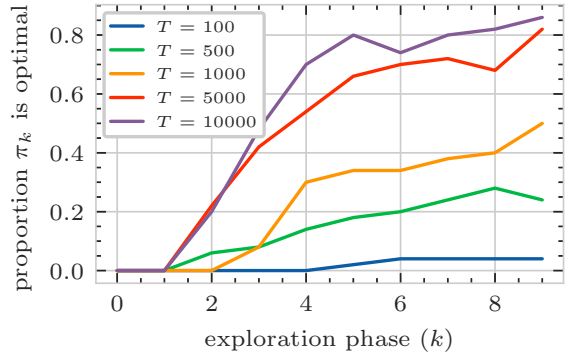}
     \caption{Simulation results: proportion of 50 trials where the policy at the $k$th exploration phase was optimal}
     \label{fig:sim_results}
 \end{figure}
 
\section{Conclusion}
\label{sec:conclusion}
In this paper, we proposed a decentralized multi-agent reinforcement learning algorithm for stochastic games with continuous spaces. We showed that as the resolution of each agent's quantizer increases, their policy-updates are asymptotically optimal responses to their most-recently observed environment. We then analyzed the global policy-updating dynamics of the general class of best-reply-based algorithms and derived a closed-form expression for the probabilities of convergence to each equilibrium. We used this to analyze the convergence of the continuous-space algorithm to equilibrium. 

\bibliography{references, SerdardBibliography,more_references}

\end{document}